\documentclass[acmsmall]{acmart}
\renewcommand\footnotetextcopyrightpermission[1]{} 
\usepackage{algorithm}
\usepackage{algpseudocode}

\AtBeginDocument{%
  \providecommand\BibTeX{{%
    \normalfont B\kern-0.5em{\scshape i\kern-0.25em b}\kern-0.8em\TeX}}}

\acmConference[Uploaded to arXiv]{}{Nov}{2020}
\acmBooktitle{a}
\acmPrice{15.00}
\acmISBN{978-1-4503-XXXX-X/18/06}

\begin{document}

\title{Multi-Agent Decentralized Belief Propagation on Graphs}


\author{Yitao Chen}
\affiliation{%
  \institution{Qualcomm}
  \streetaddress{}
  \city{San Diego}
  \country{USA}}
\email{yitachen@qti.qualcomm.com}

\author{Deepanshu Vasal}
\affiliation{%
  \institution{Northwestern University}
  \city{Evanston}
  \country{USA}
}
\email{dvasal@umich.edu}







\begin{abstract}
  We consider the problem of interactive partially observable Markov decision processes (I-POMDPs), where the agents are located at the nodes of a communication network. Specifically, we assume a certain message type for all messages. Moreover, each agent makes individual decisions based on the interactive belief states, the information observed locally and the messages received from its neighbors over the network.

Within this setting, the collective goal of the agents is to maximize the globally averaged return over the network through exchanging information with their neighbors. We propose a decentralized belief propagation algorithm for the problem, and prove the convergence of our algorithm. Finally we show multiple applications of our framework. Our work appears to be the first study of decentralized belief propagation algorithm for networked multi-agent I-POMDPs.
\end{abstract}

\maketitle

\section{Introduction}
In reinforcement learning, Partially Observable Markov Decision Processes (POMDPs) \cite{kaelbling1998planning, boutilier1999decision, russell2002artificial} is a general decision-theoretic framework for planning under uncertainty in a partially observable, stochastic environment. An agent makes decisions rationally in such settings by maintaining beliefs of the physical state and sequentially choosing the optimal actions that maximize the expected value of future rewards. Solutions of POMDPs are mappings from an agent’s beliefs to actions. The drawback of POMDPs in the multi-agent scenario is that the impact of other agents' actions cannot be represented explicitly. Examples of such POMDPs are infinite generalized policy representation \cite{liu2011infinite}, and infinite POMDPs \cite{doshi2013bayesian}. 

Interactive POMDPs (I-POMDPs) \cite{gmytrasiewicz2005framework} generalize POMDPs to multi-agent settings by substituting POMDP belief spaces with interactive belief spaces. More specifically, an I-POMDP substitutes plain beliefs of the state space with augmented beliefs of the state space and the other agents' beliefs and models. The models of other agents included in the new augmented belief space consist of two types: intentional models and subintentional models. An intentional model ascribes beliefs, preferences, and rationality to other agents, while a simpler subintentional model, such as finite state controllers \cite{panella2016bayesian} does not. The augmentation of intentional models forms a hierarchical belief structure that represents an agent's beliefs of the physical state, beliefs of the other agents and their beliefs of others' beliefs, and it can be nested to arbitrary levels. Solutions of I-POMDPs map an agent's belief of the environment and other agents' models to actions. It has been shown that the augmented belief in I-POMDP results in a higher value function compared to one obtained from POMDP, which implies I-POMDPs' modeling superiority.

In this work, we study the problem of POMDPs with collaborative agents. For collaborative POMDPs problem, it is important to specify the interactions bewteen the agents. One appealing option is to have a central controller which has the information of all agents, and determines the actions for all agents. With all the information available to the controller, the problem reduces to a classical POMDP and can be solved by existing single-agent POMDP algorithms. Yet, in a bunch of real-world applications, such as sensor networks \cite{akyildiz2002wireless, rabbat2004distributed} and intelligent transportation systems \cite{adler2002cooperative}, it may be very costly to have a central controller. Moreover, since the central controller needs to communicate with each agent to get information, it increases the communication overhead at each controller. The communication overhead degrades the scalability of the multi-agent system as well as its robustness to malicious attacks. Another option is to make the collaboration between agents implicitly represented by the I-POMDP model. Although agents' actions do not change the other agents' model directly, they can change the other agents’ belief states indirectly, typically by changing the environment in a way observable to the other agents. The influence to the other agents' belief states can be viewed as some form of exchange of information. However, the modeling superiority of I-POMDPs comes at the cost of a drastic increase of the belief space complexity, because the agent models grow exponentially as the belief nesting level increases. Hence, the complexity of the belief representation is proportional to belief dimensions, which is known as the curse of dimensionality. Moreover, due to the fact that exact solutions to POMDPs are PSPACE-complete and undecidable for finite and infinite time horizon respectively \cite{papadimitriou1987complexity}, the time complexity of more generalized I-POMDPs is at least PSPACE-complete and undecidable for finite and infinite horizon, since an I-POMDPs may contain multiple POMDPs or I-POMDPs of other agents.

Given all these disadvantages of the centralized model and the complexity of I-POMDPs, we consider a decentralized model where the agents in I-POMDPs are connected by a communication network. Specifically, let $\{\mathcal{G}=(\mathcal{V}, \mathcal{E})\}$ be a communication network, where $\mathcal{V}$ is the set of nodes, and $\mathcal{E} \subseteq \{(i,j): i,j\in \mathcal{V}\}$ is the set of edges. We assume that each node represents an agent. And agent $i \in \mathcal{V}$ and agent $j \in \mathcal{V}$ can communicate with each other if and only if $(i,j) \in \mathcal{V}$. As such, at each time slot $t$, each agent executes an individual action based on the interactive belief states, the local information and the messages sent from its neighbors, with the goal of maximizing the individual average rewards. The message type is crucial in reducing the complexity of I-POMDPs, for example, if the message is a belief, the agent does not have to maintain complex and infinitely nested beliefs of other agents. We call this model as \emph{networked multi-agent I-POMDPs}, which is presented in Section~\ref{sec:def} in detail.

 \textbf{Main Contribution.} Our contribution in this work is three-fold. First, we formulate interactive POMDPs problem for networked agents, and prove a version of belief update and value iteration update adapted to this setting. Second, we propose a decentralized belief propagation algorithm, and prove the convergence of the proposed algorithm. Third, we show our framework precisely captures the collaboration in decentralized multi-agent cooperative systems by showing multiple applications. 
 
 \textbf{Related Work.} AI literature \cite {bernstein2002complexity, nair2003taming} appeared in a series of studies that extended POMDP to several branches. One of the branches is called decentralized POMDP (DEC-POMDP), which is related to decentralized control problems \cite{ooi1996decentralized}. DEC-POMDP framework assumes that the agents have the common reward function. Furthermore, it is assumed that the optimal joint solution is computed in a central coordinator and then distributed among the agents for execution. \cite{nayyar2013decentralized} shows a variant of DEC-POMDP with a partial historical shared information structure. The framework of I-POMDPs is introduced in \cite{gmytrasiewicz2005framework}, followed by Bayesian inference approximate solutions \cite{han2018learning}. Another branch of extending POMDPs to multiple agents is called multi-agent POMDP (MPOMDP) \cite{messias2011efficient, amato2015scalable}. The MPOMDP framework assumes that agents have joint observations, so it can be simplified to POMDP \cite {pynadath2002communicative} by having a single centralized controller that takes joint actions and receives joint observations. In other words, a DEC-POMDP can be reduced to an MPOMDP in a fully communicative scenario.
 
 From the game-theoretic side, most existing works are based on the framework of Markov games, which was first proposed by \cite{littman1994markov}, and then followed by \cite{littman2001value, lauer2000algorithm, hu2003nash}. This framework applies to the setting with both collaborative and competitive relationships among agents. More recently, several multi-agent reinforcement learning (MARL) algorithms using deep neural networks as function approximators have gained increasing attention \cite{foerster2016learning, gupta2017cooperative, bahdanau2016actor, omidshafiei2017deep, foerster2017stabilising}. A more relevant work is \cite{zhang2018fully}, the authors study a MARL framework with networked agents, where the communication among agents contributes toward the overall performance of MARL in a fully decentralized setting.
 
 The remainder of this paper is structured as follows. We start with an overview of partially observable Markov decision processes, the concept of agent types and interactive POMDPs in Section~\ref{sec:bg}. We formulate the networked multi-agent I-POMDPs and present a decentralized belief propagation algorithms for the networked multi-agent I-POMDPs problem in Section~\ref{sec:def}. We provide theoretical result in Section~\ref{sec:prf}. Several applications of our framework is presented in Section~\ref{sec:app}. We conclude with a brief summary and some future research directions in Section~\ref{sec:con}.
 
 \section{Background}\label{sec:bg}
\subsection{Partially Observable Markov Decision Process}
A partially observable Markov decision process (POMDP) of an agent $i$ is defined as
\begin{flalign}
POMDP_i = \langle S, A_i, T_i, \Omega_i, O_i, R_i\rangle
\end{flalign}
where: $S$ is a set of possible states of the environment. $A_i$ is a set of agent $i$'s actions. $T_i$ is a transition function, i.e., $T_i: S \times A_i \times S \rightarrow [0,1]$ which describes the dynamics of the environment. $\Omega_i$ is the set of agent $i$'s observations. $O_i$ is agent $i$'s observation function, i.e., $O_i: S \times A_i \times \Omega_i \rightarrow [0,1]$. $R_i$ is the reward function for the agent $i$, i.e., $R_i : S \times A_i \rightarrow \mathfrak{R}$.


The belief update step of POMDP is shown below:
\begin{flalign}
b_i^t(s^t) = \beta O_i(o_i^t, s^t, a_i^{t-1}) \sum_{s^{t-1} \in S} b_i^{t-1}(s^{t-1}) T_i(s^t, a_i^{t-1}, s^{t-1})
\end{flalign}
where $\beta$ is the normalizing constant. The belief update step takes into account changes in initial belief $b_i^{t-1}$, action $a_i^{t-1}$, and the new observation $o_i^t$. And $b_i^t(s^t)$ is the new belief for state $s^t$. It is convenient for us to denote the above belief update step for all states in $S$ as $b_i^t = SE(b_i^{t-1}, a_i^{t-1}, o_i^t)$.

We denote agent $i$'s optimality criterion as $OC_i$, which specifies how rewards over time are handled. In this work, we concentrate on the infinite horizon criterion with discounting, i.e., the agent maximizes the expected value of the sum of the discounted rewards of an infinite horizon $E(\sum_{t=0}^{\infty}\gamma^t r_t)$, where $0 \le \gamma \le 1$ is a discount factor. However, our approach can be easily extended to the other criteria.

The utility associated with a belief state, $b_i$ consists of the maximum immediate rewards due to $b_i$, together with the discounted expected sum of utilities associated with the updated belief states $SE_i(b_i, a_i, o_i)$:
\begin{flalign}\label{equ:SE}
U(b_i) &= \max_{a_i \in A_i} \bigg\{ \sum_{s \in S} b_i(s) R_i(s,a_i) + \gamma \sum_{o_i \in \Omega_i} Pr(o_i\rvert a_i, b_i)U(SE_i(b_i, a_i, o_i))\bigg\}
\end{flalign}

And the optimal action, $a_i^*$, is an element of the set of optimal actions, $OPT(b_i)$, for the belief state $b_i$, defined as:
\begin{flalign}
OPT(b_i) &= \arg \max_{a_i \in A_i} \bigg\{ \sum_{s \in S} b_i(s) R_i(s,a_i) + \gamma \sum_{o_i \in \Omega_i} Pr(o_i\rvert a_i, b_i)U(SE_i(b_i, a_i, o_i))\bigg\}
\end{flalign}
\subsection{Agent Types and Frames}
The following two definitions collect POMDP parameters independent of agent implementation and put them into constructs. The representations are convenient for our analysis, so we list them below,

\begin{definition}[Type\cite{gmytrasiewicz2005framework}]
A type of an agent $i$ is, $\theta_i = \langle b_i, A_i, \Omega_i, T_i, O_i, R_i, OC_i\rangle$, where $b_i$ is agent $i$'s state of belief (an element of $\Delta(S)$), $OC_i$ is its optimality criterion, and the rest of the elements are as defined before. Let $\Theta_i$ be the set of agent $i$'s types.
\end{definition}

Given type $\theta_i$, and the assumption that the agent is Bayesian-rational, we denote the set of agent's optimal actions as $OPT(\theta_i).$


\begin{definition}[Frame\cite{gmytrasiewicz2005framework}]
A frame of an agent $i$ is, $\hat{\theta_i} = \langle A_i, \Omega_i, T_i, O_i, R_i, OC_i\rangle$. Let $\hat{\Theta}_i$ be the set of agent $i$'s frames.
\end{definition}

For brevity, we write a type as consisting of an agent's belief together with its frame: $\theta_i = \langle b_i, \hat{\theta}_i \rangle$.

\subsection{Interactive POMDPs}
W.l.o.g., we consider an agent $i$ interacting with only one other agents $j$.

\begin{definition}[I-POMDP]
An interactive POMDP of agent $i$, $\text{I-POMDP}_i$, is:
\begin{flalign}
\text{I-POMDP}_{i,l} = \langle IS_{i,l}, A, T_i, \Omega_i, O_i, R_i \rangle
\end{flalign}
\end{definition}
where $IS_{i,l}$ is a set of interactive states for agent $i$, defined as $IS_{i,l} = S \times M_{j,l-1}$, $l \ge 1$. Here $M_j$ is the the set of possible models of agent $j$, and $l$ is the nesting level. The set of models $M_{j,l-1}$ can be divided into two classes, the intentional models $IM_{j,l-1}$, and subintentional models $SM_{j}$. Thus, $M_{j,l-1} = IM_{j,l-1}  \cup  SM_{j}$.

The intentional models, $IM_{j,l-1}$, ascribe to the other agent beliefs, preferences and rationality in action selection. Thus they are other agents' types. For example, the intentional model for agent $j$ at level $l-1$ can be defined as $\theta_{j,l-1} = \langle b_{j,l-1}, \hat{\theta}_j\rangle$, where $b_{j,l-1}$ is agent $j$'s belief nested to the level $l-1$, $b_{j,l-1} \in \Delta(IS_{j,l-1})$. We omit the details of subintentional model since it is not a focus of this paper.

The interactive states $IS_{i,l}$ can be defined in an inductive manner:
\begin{flalign}
    & IS_{i,0}=S \quad \quad \quad \ \ \ \Theta_{j,0}=\{\langle b_{j,0}, \hat{\theta}_j\rangle:b_{j,0} \in \Delta(IS_{j,0})\} \quad M_{j,0}=\Theta_{j,0} \cup SM_j  \nonumber \\
    & IS_{i,1}=S \times M_{j,0}\quad \Theta_{j,1}=\{\langle b_{j,1}, \hat{\theta}_j\rangle:b_{j,1} \in \Delta(IS_{j,1})\} \quad M_{j,1}=\Theta_{j,1} \cup M_{j,0}  \nonumber\\
    & \cdots \\
    & IS_{i,l}=S \times M_{j,l-1}\quad \Theta_{j,l}=\{\langle b_{j,l}, \hat{\theta}_j\rangle:b_{j,l} \in \Delta(IS_{j,l})\} \quad M_{j,l}=\Theta_{j,l} \cup M_{j,l-1} \nonumber.
\end{flalign}

For the rest of the paper, we omit the level subscription for notation simplicity.

All remaining components in an I-POMDP are similar to those in a POMDP except we also keep the following assumptions in \cite{gmytrasiewicz2005framework}:

\textbf{Model Non-manipulability Assumption (MNM):} Agents' actions do not change the other agents' models directly. 

\textbf{Model Non-observability (MNO):} Agents cannot observe other's models directly.
Next, we define the belief update steps and value iterations for I-POMDPs.

\subsubsection{Belief Update in I-POMDPs}


The next proposition defines the agent $i$'s belief update function, $b_i^t(is^t) = Pr(is^t \rvert o_i^t, a_i^{t-1}, b_i^{t-1})$, where $is^t \in IS_i$ is an interactive state. We use the belief state estimation function, $SE_{\theta_i}$, as an abbreviation for belief updates for individual states so that 
\begin{flalign}
b_i^t = SE_{\theta_i}(b_i^{t-1}, a_i^{t-1}, o_i^t).
\end{flalign}
$\tau_{\theta_i}(b_i^{t-1}, a_i^{t-1}, o_i^t, b_i^t)$ will stand for $Pr(b_i^t \rvert b_i^{t-1}, a_i^{t-1}, o_i^t)$. Further below we also define the set of type-dependent optimal actions of an agent, $OPT(\theta_i)$.

Given the definition and assumptions, the interactive belief update can be performed as in the next proposition.
\begin{proposition}[Belief Update]\label{prop:SE-IPOMDP} Under the MNM and MNO assumptions, the belief update function for an interactive POMDP $\langle IS_i, A, T_i, \Omega_i, O_i, R_i \rangle$, when $m_j$ in $is^t$ is intentional, is:
\begin{flalign}\label{equ:SE-IPOMDP}
b_i^t(is^t) =& \beta \sum_{is^{t-1}:\hat{m}_j^{t-1}=\hat{\theta}_j^t} b_i^{t-1}(is^{t-1}) \sum_{a_j^{t-1}} Pr(a_j^{t-1} \rvert \theta_j^{t-1}) O_i(s^t, a^{t-1}, o_i^t)\times T_i(s^{t-1}, a^{t-1}, s^t) \nonumber\\
&\sum_{o_j^t} \tau_{\theta_j^t}(b_j^{t-1}, a_j^{t-1}, o_j^t, b_j^t) O_j(s^t, a^{t-1}, o_j^t). 
\end{flalign}
\end{proposition}


\subsubsection{Value Function and Solution in I-POMDPs}
Analogously to POMDPs, each belief state in I-POMDP has an associated value reflecting the maximum reward the agent can expect in this belief state:
\begin{flalign}\label{equ:VI}
U(\theta_i) &= \max_{a_i \in A_i} \bigg\{ \sum_{is} b_i(is) ER_i(is,a_i) + \gamma \sum_{o_i \in \Omega_i} Pr(o_i\rvert a_i, b_i)U(\langle SE_{\theta_i}(b_i, a_i, o_i), \hat{\theta}_i)\bigg\}
\end{flalign}
where, $ER_i(is, a_i) = \sum_{a_j} R_i(is, a_i, a_j) Pr(a_j \rvert m_j)$. Equation~(\ref{equ:VI}) is a basis for value iteration in I-POMDPs.

Agent $i$'s optimal action, $a_i^*$, for the case of infinite horizon criterion with discounting, is an element of the set of optimal actions for the belief state, $OPT(\theta_i)$, defined as 
\begin{flalign}\label{equ:OA}
OPT(\theta_i) &= \arg \max_{a_i \in A_i} \bigg\{ \sum_{is} b_i(is) ER_i(is,a_i) + \gamma \sum_{o_i \in \Omega_i} Pr(o_i\rvert a_i, b_i)U(\langle SE_{\theta_i}(b_i, a_i, o_i), \hat{\theta}_i)\bigg\}.
\end{flalign}

\section{Networked Multi-Agent I-POMDP}\label{sec:def}
Consider a set $\mathcal{V}$ of $N$ agents, labeled by an index $i = 1,2,...,N$. Their interaction (communication) is modeled by a graph $\mathcal{G} = (\mathcal{V},\mathcal{E})$ where an edge $(i,j)$ is in $\mathcal{E}$ if and only if agent $i$ interacts (communicates) with agent $j$. We assume that $\mathcal{G}$ is connected, i.e., there is a path from any node $i$ to any other node $j$. We denote the set of neighbors of agent $i$ as $\partial i$. The network induced by interaction and communication could be different, but here we do no distinguish them for simplicity.

Let us define type of message and message before we proceed to define the networked multi-agent I-POMDPs.
\begin{definition}[Message Type $\&$ Message]
We define the set of message types as $\mathcal{M}$ $=$ $\{\times_{i \in [N]} A_i,$ $ \Delta(S), \times_{i \in [N]} \Omega_i\}$, i.e., there are three types of messages, `action', `belief state' and `observation'. A type of message $M$ is an element of the message type set $\mathcal{M}$, i.e., $M \in \mathcal{M}$. A message $\mu_i^t$ is an element of message type $M$, i.e. $\mu_i^t \in M$.
\end{definition}
For example, a message sent by an agent $i$ at time slot $t$ of type `action' can be represented by $\mu_i^t = a_j^{t-1}$.

Now we can define the networked multi-agent I-POMDP.
\begin{definition}[Networked Multi-agent I-POMDP]\label{NMA-I-POMDP}
A networked multi-agent I-POMDP is characterized by a tuple $\langle \mathcal{G}, \{IS_i\}_{i\in [N]},\{A_i\}_{i \in [N]}, \{T_i\}_{i \in [N]}, \{\Omega_i\}_{i \in [N]}$, $\{O_i\}_{i \in [N]},\{R_i\}_{i \in [N]}, M \rangle$.
\end{definition}
\begin{itemize}
    \item $\mathcal{G}$ is the communication network.
    \item For each agent $i$, the interactive state $IS_i$ is defined as $IS_i = S \times M_{\partial i}$, where $\partial i$ is the set of neighbors of agent $i$ in the communication network $\mathcal{G}$.
    \item For each agent $i$, the action space is $A_i$.
    \item For each agent $i$, under the MNM assumption, the transition model is defined as $T_i: S \times A_i \times A_{\partial i} \times S \rightarrow [0,1]$.
    \item For each agent $i$, $\Omega_i$ is defined as before in the I-POMDP model.
    \item For each agent $i$, under the MNO assumption, the observation transition function is defined as $O_i: S \times A_i \times A_{\partial i} \times \Omega_i \rightarrow [0,1]$.
    \item For each agent $i$, $R_i$ is defined as $IS_i \times A_i \times A_{\partial i} \rightarrow \mathfrak{R}$ 
    \item $M$ is the message type.
\end{itemize}

Similar to I-POMDPs, we can define agent $i$'s belief update function. For simplicity, from now on we assume $\partial i = \{j\}$ (we can extend it to more than one neighbor case easily), the belief update function for agent $i$ is $b_i^t(is^t) = Pr(is^t \rvert o_i^t, a_i^{t-1}, b_i^{t-1}, \mu_j^t)$, where $is^t \in IS_i$ is an interactive state and $\mu_j^t$ is the message sent by an agent $i$ at time slot $t$. We use the belief state estimation function $SE_{\theta_i}$, as an abbreviation for belief updates for individual states so that 
\begin{flalign}
b_i^t = SE_{\theta_i}(b_i^{t-1}, a_i^{t-1}, o_i^t, \mu_j^t).
\end{flalign}
The  next  proposition  defines  the  agent $i$'s  belief  update function in detail.

\begin{proposition}[Belief Update]\label{prop:SE-NMA-I-POMDP} Under the MNM and MNO assumptions, the belief update function for agent $i$ of a networked multi-agent I-POMDP $\langle \mathcal{G}, \{IS_i\}_{i\in [N]},\{A_i\}_{i \in [N]}, \{T_i\}_{i \in [N]}, \{\Omega_i\}_{i \in [N]},$ $\{O_i\}_{i \in [N]}$, $\{R_i\}_{i \in [N]}, M \rangle$, when $m_j$ in $is^t$ is intentional, is:
\begin{itemize}
\item When message type is `action':
\begin{flalign}
b_i^t(is^t) =& \beta \sum_{is^{t-1}:\hat{m}_j^{t-1}=\hat{\theta}_j^t} b_i^{t-1}(is^{t-1}) O_i(s^t, a^{t-1}, o_i^t)  T_i(s^{t-1}, a^{t-1}, s^t) \nonumber \\
&\sum_{o_j^t} \tau_{\theta_j^t}(b_j^{t-1}, a_j^{t-1}, o_j^t, b_j^t) 
O_j(s^t, a^{t-1}, o_j^t).
\end{flalign}
\item When message type is `belief state':
\begin{flalign}
b_i^t(is^t) =& \beta \sum_{is^{t-1}:\hat{m}_j^{t-1}=\hat{\theta}_j^t} b_i^{t-1}(is^{t-1}) \sum_{a_j^{t-1} \in OPT(\theta_j)}  Pr(a_j^{t-1} \rvert \theta_j^{t-1}) \nonumber \\
&O_i(s^t, a^{t-1}, o_i^t) T_i(s^{t-1}, a^{t-1}, s^t).
\end{flalign}
\item When message type is `observation':
\begin{flalign}
b_i^t(is^t) =& \beta \sum_{is^{t-1}:\hat{m}_j^{t-1}=\hat{\theta}_j^t} b_i^{t-1}(is^{t-1}) \sum_{a_j^{t-1}} Pr(a_j^{t-1} \rvert \theta_j^{t-1})  O_i(s^t, a^{t-1}, o_i^t)\nonumber \\ &T_i(s^{t-1}, a^{t-1}, s^t) \tau_{\theta_j^t}(b_j^{t-1}, a_j^{t-1}, o_j^t, b_j^t) O_j(s^t, a^{t-1}, o_j^t). 
\end{flalign}
\end{itemize}
\end{proposition}
We leave the proof of Propositions~\ref{prop:SE-NMA-I-POMDP} to Section~\ref{sec:prf}.

The value function $U(\theta_i)$ is 
\begin{flalign}\label{equ:VI-NMA}
U(\theta_i, \mu_j) &= \max_{a_i \in A_i} \bigg\{ \sum_{is} b_i(is) ER_i(is,a_i)+  \gamma \sum_{o_i \in \Omega_i} Pr(o_i\rvert a_i, b_i)U(\langle SE_{\theta_i}(b_i, a_i, o_i, \mu_j), \hat{\theta}_i)\bigg\}
\end{flalign}
where $ER_i(is, a_i) = \sum_{a_j} R_i(is, a_i, a_j) Pr(a_j \rvert m_j, \mu_j)$.

And the set of optimal actions for agent $i$ is defined as,
\begin{flalign}\label{equ:OA-NMA}
OPT(\theta_i, \mu_j) &= \arg \max_{a_i \in A_i} \bigg\{ \sum_{is} b_i(is) ER_i(is,a_i) +  \gamma \sum_{o_i \in \Omega_i} Pr(o_i\rvert a_i, b_i)U(\langle SE_{\theta_i}(b_i, a_i, o_i, \mu_j), \hat{\theta}_i)\bigg\}.
\end{flalign}

Note the Equation~(\ref{equ:VI-NMA}) can be rewritten in the following form $U^n = HU^{n-1}$. Here $H: B \rightarrow B$ is a backup operator, and is defined as,
\begin{flalign}
HU^{n-1}(\theta_i,\mu_j) = \max_{a_i \in A_i} h(\theta_i, a_i, \mu_j, U^{n-1}),
\end{flalign}
where $h: \Theta_i \times A_i \times M \times B \rightarrow \mathbb{R}$ is,
\begin{flalign}
h(\theta_i, a_i, \mu_j, U) &= \sum_{is} b_i(is) ER_i(is,a_i)+  \gamma \sum_{o_i \in \Omega_i} Pr(o_i\rvert a_i, b_i)U(\langle SE_{\theta_i}(b_i, a_i, o_i, \mu_j), \hat{\theta}_i).
\end{flalign}
\subsection{Algorithm}
Now we are ready to present the decentralized belief propagation algorithm for networked multi-agent systems. The algorithm requires each agent to maintain a belief on its interactive states, while allows each agent $i$ share messages of certain type with its neighbors on the network. In this way, each agent is able to improve the value function and thus the current policy. 

\begin{algorithm}                               
\caption{Decentralized Belief Propagation Algorithm}          
\label{alg:DecBP} 

  \begin{algorithmic} 
    \State \textbf{Input}: Initialize $b_i(is)$, $U(\theta_i)$ for all $is \in IS_i$ and for all $i \in [N]$ 
    \Repeat :
		\For{ all $i$ in $[N]$}
		    \State Observe $o_i^t$, and reward $r_i^t$.
		    \State Send message $\mu_i^t$ to all neighbors $\partial i$. 
		    \State Update the belief given the received messages 
		    \State $b_i^t = SE_{\theta_i}(b_i^{t-1}, a_i^{t-1}, o_i^t, \{\mu_j^t\}_{j \in \partial i})$. 
		    \State Update the value function
		    \State $U(\theta_i) \leftarrow HU(\theta_i)$. 
		    \State Select and execute action $a_i^t$.
		\EndFor
	\Until{convergence}
  \end{algorithmic}
\end{algorithm}

Note similar to the typical belief propagation algorithm, all agents send/receive messages simultaneously and then update their beliefs simultaneously.
\section{Theoretical Results}\label{sec:prf}
We start this section by proving Proposition~\ref{prop:SE-NMA-I-POMDP} for the case that message type is `action', i.e., $M = \times_{i \in [N]} A_i$ and more specifically $\mu_j^t = a_j^{t-1}$. The belief update step for other message types can be derived in similar ways.
\begin{proof}
We start by applying the Bayes Theorem:
\begin{flalign}\label{equ:belief-update}
b_i^t(is^t)&=Pr(is^t\rvert o_i^t, a_i^{t-1}, b_i^{t-1}, \mu_j^t) \nonumber \\
&= \frac{Pr(is^t,o_i^t\rvert a_i^{t-1}, b_i^{t-1}, \mu_j^t)}{Pr(o_i^t\rvert a_i^{t-1}, b_i^{t-1}, \mu_j^t)} \nonumber\\
&=\beta \sum_{is^{t-1}} b_i^{t-1}(is^{t-1})Pr(is^t,o_i^t\rvert a_i^{t-1}, is^{t-1}, \mu_j^t)\nonumber\\
&=\beta \sum_{is^{t-1}} b_i^{t-1}(is^{t-1}) \sum_{a_j^{t-1}} Pr (a_j^{t-1}\rvert a_i^{t-1}, is^{t-1}, \mu_j^t)  Pr(is^t, o_i^t\rvert a_i^{t-1}, a_j^{t-1}, is^{t-1}, \mu_j^t) \nonumber\\ 
&=\beta \sum_{is^{t-1}} b_i^{t-1}(is^{t-1}) \sum_{a_j^{t-1}} Pr (a_j^{t-1}\rvert is^{t-1}, \mu_j^t) Pr(is^t, o_i^t\rvert a_i^{t-1}, a_j^{t-1}, is^{t-1}, \mu_j^t) \nonumber\\ 
&=\beta \sum_{is^{t-1}} b_i^{t-1}(is^{t-1}) \sum_{a_j^{t-1}} Pr (a_j^{t-1}\rvert b_j^{t-1}, \mu_j^t) Pr(o_i^t\rvert a^{t-1}, is^{t}, is^{t-1}, \mu_j^t) Pr(is^t\rvert a^{t-1}, is^{t-1}, \mu_j^t) \nonumber\\ 
&\overset{(a)}{=}\beta \sum_{is^{t-1}} b_i^{t-1}(is^{t-1})Pr(o_i^t\rvert a^{t-1}, is^{t}) Pr(is^t\rvert a^{t-1}, is^{t-1}) \nonumber\\ 
&=\beta \sum_{is^{t-1}} b_i^{t-1}(is^{t-1}) O_i(s^t, a^{t-1}, o_i^t) Pr(is^t\rvert a^{t-1}, is^{t-1}), 
\end{flalign}
where the equality $(a)$ holds because $Pr(a_j^{t-1}\rvert b_j^{t-1}, \mu_j^t)=1$ when $\mu_j^t=a_j^{t-1}$ and $Pr(a_j^{t-1}\rvert b_j^{t-1}, \mu_j^t)=0$ otherwise. And recall $\mu_j^t = a_j^{t-1}$ is our assumption.

Since we assume the interactive states are intentional, $is^t =(s^t, \theta_j^t) = (s^t, b_j^t, \hat{\theta}^t_j)$, we can simplify the term $Pr(is^t\rvert a^{t-1}, is^{t-1})$.
\begin{flalign}\label{equ:is-update}
&Pr(is^t\rvert a^{t-1}, is^{t-1}) \nonumber\\
=& Pr(s^t, b_j^t, \hat{\theta}^t_j\rvert a^{t-1}, is^{t-1}) \nonumber\\
=& Pr(b_j^t\rvert s^t, \hat{\theta}^t_j , a^{t-1}, is^{t-1})Pr(s^t, \hat{\theta}^t_j\rvert a^{t-1}, is^{t-1})\nonumber\\
=& Pr(b_j^t\rvert s^t, \hat{\theta}^t_j , a^{t-1}, is^{t-1}) Pr(\hat{\theta}^t_j\rvert s^t, a^{t-1}, is^{t-1})Pr(s^t\rvert a^{t-1}, is^{t-1}) \nonumber\\
=& Pr(b_j^t\rvert s^t, \hat{\theta}^t_j , a^{t-1}, is^{t-1}) I(\hat{\theta}^t_j, \hat{\theta}^{t-1}_j) T_i(s^{t-1},a^{t-1},s^t),
\end{flalign}
where $I(\cdot, \cdot)$ is a boolean identity function, which equal $1$ if the two frames are identical, and $0$ otherwise. The joint action pair, $a^{t-1}$, may change the physical state. The third term on the right-hand side of Equation~(\ref{equ:is-update}) captures this transition.
\begin{flalign}\label{equ:bj-update}
&Pr(b_j^t\rvert s^t, \hat{\theta}^t_j , a^{t-1}, is^{t-1}) \nonumber \\
=& \sum_{o_j^t} Pr(b_j^t\rvert s^t, \hat{\theta}^t_j , a^{t-1}, is^{t-1}, o_j^t)Pr(o_j^t\rvert s^t, \hat{\theta}^t_j , a^{t-1}, is^{t-1}) \nonumber \\
=& \sum_{o_j^t} Pr(b_j^t\rvert s^t, \hat{\theta}^t_j , a^{t-1}, is^{t-1}, o_j^t)Pr(o_j^t\rvert s^t, \hat{\theta}^t_j , a^{t-1}) \nonumber \\
=& \sum_{o_j^t} \tau_{\theta_j^t}(b_j^{t-1}, a_j^{t-1}, o_j^t, b_j^t)O_j(s^t, a^{t-1}, o_j^t).
\end{flalign}
In Equation~(\ref{equ:bj-update}), the first term on the right-hand side is $1$ if agent $j$'s belief update, $SE_{\theta_j}(b_j^{t-1}, a_j^{t-1}, o_j^t)$ generates a belief state equal to $b_j^t$. In the second terms on the right-hand side of the equation, the MNO assumption allows us to replace $Pr(o_j^t\rvert s^t, \hat{\theta}^t_j , a^{t-1})$ with $O_j(s^t, a^{t-1}, o_j^t)$.

Let us substitute Equation~(\ref{equ:bj-update}) into Equation~(\ref{equ:is-update}),
\begin{flalign}\label{Equ:ISUpdate}
Pr(is^t\rvert a^{t-1}, is^{t-1}) 
= \sum_{o_j^t} \tau_{\theta_j^t}(b_j^{t-1}, a_j^{t-1}, o_j^t, b_j^t)O_j(s^t, a^{t-1}, o_j^t) I(\hat{\theta}^t_j, \hat{\theta}^{t-1}_j) T_i(s^{t-1},a^{t-1},s^t).
\end{flalign}

Now substitute Equation~(\ref{Equ:ISUpdate}) into Equation~(\ref{equ:belief-update}), we have
\begin{flalign}
b_i^t(is^t)=&\beta \sum_{is^{t-1}} b_i^{t-1}(is^{t-1}) \ O_i(s^t, a^{t-1}, o_i^t)  \sum_{o_j^t} \tau_{\theta_j^t}(b_j^{t-1}, a_j^{t-1}, o_j^t, b_j^t)O_j(s^t, a^{t-1}, o_j^t)\nonumber \\
&  I(\hat{\theta}^t_j, \hat{\theta}^{t-1}_j) T_i(s^{t-1},a^{t-1},s^t).
\end{flalign}

We can remove the term $I(\hat{\theta}^t_j, \hat{\theta}^{t-1}_j)$ by changing the scope of the first summation, which gives us the final expression for the belief update,
\begin{flalign}
b_i^t(is^t) =& \beta \sum_{is^{t-1}:\hat{m}_j^{t-1}=\hat{\theta}_j^t} b_i^{t-1}(is^{t-1}) O_i(s^t, a^{t-1}, o_i^t) T_i(s^{t-1}, a^{t-1}, s^t)\nonumber \\ 
&\sum_{o_j^t} \tau_{\theta_j^t}(b_j^{t-1}, a_j^{t-1}, o_j^t, b_j^t) O_j(s^t, a^{t-1}, o_j^t).
\end{flalign}
\end{proof}

Next, for an agent $i$ and its $\text{I-POMDP}_{i}$, following the proof idea in \cite{gmytrasiewicz2005framework}, we prove the convergence of our algorithm. First, we show some properties of the back up operator $H$,
\begin{lemma}\label{lemma:1}
For any finitely nested I-POMDP value functions $V$ and $U$, if $V \le U$ , then $HV \le HU$.
\end{lemma}
\begin{proof}
Select arbitrary value functions $V$ and $U$ such that $V(\theta_{i}, \mu_V) \le U(\theta_{i}, \mu_U)$, $\forall \theta_{i} \in \Theta_{i}$, $\mu_V, \mu_U \in M$, where $\theta_{i}$ is an arbitrary type of agent $i$ and $\mu_V, \mu_U$ are arbitrary messages.
\begin{flalign}
&HV(\theta_{i}, \mu_V) \nonumber \\
=& \max_{a_i \in A_i} \bigg\{ \sum_{is} b_i(is) ER_i(is,a_i)+ \gamma \sum_{o_i \in \Omega_i} Pr(o_i\rvert a_i, b_i)V(\langle SE_{\theta_{i}}(b_i, a_i, o_i, \mu_V), \hat{\theta}_{i})\bigg\} \nonumber\\
& = \sum_{is} b_i(is) ER_i(is,a_i^*)+ \gamma \sum_{o_i \in \Omega_i} Pr(o_i\rvert a_i^*, b_i)V(\langle SE_{\theta_{i}}(b_i, a_i^*, o_i, \mu_V), \hat{\theta}_{i}) \nonumber\\
& \le  \sum_{is} b_i(is) ER_i(is,a_i^*)+ \gamma \sum_{o_i \in \Omega_i} Pr(o_i\rvert a_i^*, b_i)U(\langle SE_{\theta_{i}}(b_i, a_i^*, o_i, \mu_U), \hat{\theta}_{i}) \nonumber\\
& \le \max_{a_i \in A_i} \bigg\{ \sum_{is} b_i(is) ER_i(is,a_i)+ \gamma \sum_{o_i \in \Omega_i} Pr(o_i\rvert a_i, b_i)U(\langle SE_{\theta_{i}}(b_i, a_i, o_i, \mu_U), \hat{\theta}_{i})\bigg\} \nonumber\\
& = HU(\theta_{i}, \mu_U). 
\end{flalign}
Since $\theta_{i}, \mu_V, \mu_U $ are arbitrary, $HU \le HV$.
\end{proof}
\begin{lemma}\label{lemma:2}
For any finitely nested I-POMDP value functions $V,U$, and a discount factor $\gamma \in (0,1)$, $\lVert HV - HU \rVert \le \gamma \lVert V-U \rVert$.
\end{lemma}
\begin{proof}
Assume two arbitrary well defined value functions $V$ and $U$ such that $V \le U$. From Lemma~\ref{lemma:1}, it follows that $HV \le HU$. Let $\theta_{i}$ be an arbitrary type of agent $i$ and $\mu_V, \mu_U$ be arbitrary messages. And let $a_i^*$ be the optimal action of $HU(\theta_{i}, \mu_U)$, we have,
\begin{flalign}
0 &\le HV(\theta_{i}, \mu_V) - HU(\theta_{i}, \mu_U) \nonumber\\
& = \max_{a_i \in A_i} \bigg\{ \sum_{is} b_i(is) ER_i(is,a_i)+  \gamma \sum_{o_i \in \Omega_i} Pr(o_i\rvert a_i, b_i)V(\langle SE_{\theta_{i}}(b_i, a_i, o_i, \mu_V), \hat{\theta}_{i})\bigg\} \nonumber\\
& - \max_{a_i \in A_i} \bigg\{ \sum_{is} b_i(is) ER_i(is,a_i)+  \gamma \sum_{o_i \in \Omega_i} Pr(o_i\rvert a_i, b_i)U(\langle SE_{\theta_{i}}(b_i, a_i, o_i, \mu_U), \hat{\theta}_{i})\bigg\} \nonumber\\
& \le \sum_{is} b_i(is) ER_i(is,a_i^*)+ \gamma \sum_{o_i \in \Omega_i} Pr(o_i\rvert a_i^*, b_i)V(\langle SE_{\theta_{i}}(b_i, a_i^*, o_i, \mu_V), \hat{\theta}_{i}) \nonumber\\
& - \sum_{is} b_i(is) ER_i(is,a_i^*)- \gamma \sum_{o_i \in \Omega_i} Pr(o_i\rvert a_i^*, b_i)U(\langle SE_{\theta_{i}}(b_i, a_i^*, o_i, \mu_U), \hat{\theta}_{i}) \nonumber\\
& = \gamma \sum_{o_i \in \Omega_i} Pr(o_i\rvert a_i^*, b_i) \bigg\{ V(\langle SE_{\theta_{i}}(b_i, a_i^*, o_i, \mu_V), \hat{\theta}_{i}) - U(\langle SE_{\theta_{i}}(b_i, a_i^*, o_i, \mu_U), \hat{\theta}_{i}) \bigg\} \nonumber \\
& = \gamma \sum_{o_i \in \Omega_i} Pr(o_i\rvert a_i^*, b_i) \lVert V-U \rVert \nonumber \\
& = \gamma \lVert V-U \rVert.
\end{flalign}
As the supremum norm is symmetrical, a similar result can be derived for $HU(\theta_{i}, \mu_U) - HV(\theta_{i}, \mu_V)$. Since $\theta_{i}, \mu_V, \mu_U $ are arbitrary, we prove the lemma.
\end{proof}
Based on Lemma~\ref{lemma:1} and Lemma~\ref{lemma:2}, following the Contraction Mapping Theorem in \cite{stokey1989recursive}, we can prove for each agent $i$, the value iteration in its $\text{I-POMDP}_{i}$ converges to a unique fixed point. We state the Contraction Mapping Theorem \cite{stokey1989recursive} below
\begin{theorem}[Contraction Mapping Theorem\cite{stokey1989recursive}]\label{thm:CMT}
If $(S, \rho)$ is a complete metric space and $T: S \rightarrow S$ is a contraction mapping with modulus $\gamma$, then
\begin{itemize}
    \item[1.] $T$ has exactly one fixed point $U^*$ in $S$, and
    \item[2.] The sequence $\{U^n\}$ converges to $U^*$.
\end{itemize}
\end{theorem}
\begin{theorem}
For a networked multi-agent I-POMDP, Algorithm~\ref{alg:DecBP} converges if the value functions of all agents are well defined.
\end{theorem}
\begin{proof}
First, the normed space $(B, \lVert \cdot \rVert)$ is complete w.r.t. the metric induced by the supremum norm. Second, Lemma~\ref{lemma:2} proves the contraction property of the operator $H$. Directly applying Theorem~\ref{thm:CMT}, letting $T=H$, we prove the value iteration in I-POMDPs converges to a unique fixed point.
\end{proof}
And we naturally have the following theorem.
\begin{theorem}\label{thm:OP}
For a networked multi-agent I-POMDP, the optimal policies for agent $i$, $i \in [N]$ is given by Equation~(\ref{equ:OA-NMA}).
\end{theorem}

\section{Applications}\label{sec:app}
In this section, we show our networked multi-agent I-POMDPs framework can be applied to various applications.
\subsection{Decentralized Control Problem}
Let us consider the partial history sharing information model in \cite{nayyar2013decentralized}. Consider a dynamic system with $N$ controllers. The system operates in discrete time for a horizon $T$. Let $X^{\text{($t$)}} \in \mathcal{X}^{\text{($t$)}}$ denote the state of the system at time $t$, $U^{\text{($t$)}}_i \in \mathcal{U}^{\text{($t$)}}_i$ denote the control action of controller $i$, $i \in [N]$ at time $t$, and $\mathbf{U}^{\text{($t$)}}$ denote the vector $(U^{\text{($t$)}}_1, \ldots, U^{\text{($t$)}}_N)$. The initial state $X^{\text{($1$)}}$ has a probability distribution $Q^{\text{($1$)}}$ and evolves according to
\begin{flalign}
X^{\text{($t$+$1$)}} = f^{\text{($t$)}}(X^{\text{($t$)}}, \mathbf{U}^{\text{($t$)}}, W^{\text{($t$)}}_0),
\end{flalign}
where $\{W^{\text{($t$)}}_0\}_{t=1}^{T}$ is a sequence of i.i.d. random variables with probability distribution $Q_{W,0}$.

At any time $t$, each controller has access to three types of data: current observation, local memory, and shared memory.
\begin{itemize}
    \item Current local observation: Each controller makes a local observation $Y^{\text{($t$)}}_i \in \mathcal{Y}^{\text{($t$)}}_i$ on the state of the system at time $t$,
    \begin{flalign}
        Y^{\text{($t$)}}_i = h^{\text{($t$)}}_i(X^{\text{($t$)}}, W^{\text{($t$)}}_i),
    \end{flalign}
    where $\{W^{\text{($t$)}}_i\}_{t=1}^{T}$ is a sequence of i.i.d. random variables with probability distribution $Q_{W,i}$. We assume that the random variables in the collection $\{X^{\text{($1$)}}, W^{\text{($t$)}}_j, t=1,\ldots, T, j=0,1,\ldots, N\}$ are mutually independent.
    \item Local memory: Each controller stores a subset $M^{\text{($t$)}}_i$ of its past local observations and its past actions in a local memory:
    \begin{flalign}
        M^{\text{($t$)}}_i \subset \{Y^{\text{(1:$t$)}}_i, U^{\text{(1:$t$)}}_i\}.
    \end{flalign}
    At $t=1$, the local memory is empty, $M^{\text{($t$)}}_1 = \emptyset$.
    \item Shared memory: In addition to its local memory, each controller has access to a shared memory. The contents $C_t$ of the shared memory at time $t$ are a subset of the past local observations and control actions of all controllers:
    \begin{flalign}
        C^{\text{($t$)}} \subset \{\mathbf{Y}^{\text{(1:$t$)}}, \mathbf{U}^{\text{(1:$t$)}}_i\},
    \end{flalign}
    where $\mathbf{Y}^{\text{($t$)}}$ and $\mathbf{U}^{\text{($t$)}}$ denote the vectors $(Y^{\text{($t$)}}_1, \ldots, Y^{\text{($t$)}}_N)$ and $(U^{\text{($t$)}}_1, \ldots, U^{\text{($t$)}}_N)$ respectively. At $t=1$, the shared memory is empty, $C^{\text{($1$)}} = \emptyset$.
\end{itemize}

Controller $i$ chooses action $U^{\text{($t$)}}_i$ as a function of the total data $Y^{\text{($t$)}}_i, M^{\text{($t$)}}_i, C^{\text{($t$)}}$ available to it. Specifically, for every controller $i$, $i\in [N]$,
\begin{flalign}
    U^{\text{($t$)}}_i = g^{\text{($t$)}}_i(Y^{\text{($t$)}}_i, M^{\text{($t$)}}_i, C^{\text{($t$)}}),
\end{flalign}
where $g^{\text{($t$)}}_i$ is called the control law of controller $i$. The collection $\mathbf{g}_i=(g^{\text{($1$)}}_i,\ldots, g^{\text{($T$)}}_i)$ is called the control strategy of controller $i$. The collection $\mathbf{g}_{1:N}=(\mathbf{g}_1,\ldots, \mathbf{g}_N)$ is called the control strategy of the system.

At time $t$, the system incurs a cost $l(X^{\text{($t$)}}, \mathbf{U}^{\text{($t$)}})$. The performance of the control strategy of the system is measured by the expected total cost 
\begin{flalign}\label{equ:DCP-cost}
    J(\mathbf{g}_{1:N}):= \mathbb{E}^{\mathbf{g}_{1:N}} [\sum_{t=1}^T l(X^{\text{($t$)}}, \mathbf{U}^{\text{($t$)}})],
\end{flalign}
where the expectation is with respect to the joint probability measure on $(X^{\text{(1:$T$)}}, \mathbf{U}^{\text{(1:$T$)}})$ induced by the choice of $\mathbf{g}_{1:N}$.

We are interested in the following optimization problem
\begin{definition}\label{def:DCP}
For the model described above, given the state evolution functions $f^{\text{($t$)}}$, the observation functions $h^{\text{($t$)}}_i$, the protocols for updating local and share memory, the cost function $l$, the distributions $Q^{\text{($1$)}}$, $Q_{W,i}$, $i = 0,1,\ldots, N$, and the horizon $T$, find a control strategy $\mathbf{g}_{1:N}$ for the system that minimized the expected total cost given by Equation~(\ref{equ:DCP-cost}).
\end{definition}

\cite{nayyar2013decentralized} show the decentralized system defined in Definition~\ref{def:DCP} can be viewed as a coordinated system. The coordinator only knows the shared memory $C^{\text{($t$)}}$ at time $t$. At time $t$, the coordinator chooses mappings $\Gamma^{\text{($t$)}}_i: \mathcal{Y}^{\text{($t$)}}_i\times \mathcal{M}^{\text{($t$)}}_i \rightarrow \mathcal{U}^{\text{($t$)}}_i$, for $i \in [n]$, according to 
\begin{flalign}
\mathbf{\Gamma}^{\text{($t$)}} = d^{\text{($t$)}}(C^{\text{($t$)}}, \mathbf{\Gamma}^{\text{(1:$t$-1)}}),
\end{flalign}
where $\mathbf{\Gamma}^{\text{($t$)}} = (\Gamma^{\text{($t$)}}_1, \Gamma^{\text{($t$)}}_2, \ldots, \Gamma^{\text{($t$)}}_n)$, and the function $d^{\text{($t$)}}$ is called \emph{coordination rule} at time $t$. The function $\Gamma^{\text{($t$)}}_i$ is called the \emph{coordinator’s prescription} to controller $i$. At time $t$, the function $\Gamma^{\text{($t$)}}_i$ is communicated to controller $i$, and then the controller $i$ generates an action using the function $\Gamma^{\text{($t$)}}_i$ based on its current local observation and its local memory:
\begin{flalign}
U^{\text{($t$)}}_i = \Gamma^{\text{($t$)}}_i(Y^{\text{($t$)}}_i, M^{\text{($t$)}}_i).
\end{flalign}

Moreover, the coordinated system can further be viewed as an instance of a POMDP model by defining the state process as $s^{t} :=\{X^{\text{($t$)}}, \mathbf{Y}^{\text{($t$)}}, \mathbf{M}^{\text{($t$)}} \}$, the observation process as $o^t:=C^{\text{($t$-1)}}$, and the action process $A^t := \mathbf{\Gamma}^{\text{($t$)}}$. And we can define the information state at time $t$ for the POMDP of the coordinator as:
\begin{flalign}
\Pi^{\text{($t$)}} := \mathbb{P}(s^{\text{($t$)}} \lvert C^{\text{($t$)}}, \mathbf{\Gamma}^{\text{(1:$t$)}}).
\end{flalign}

Furthermore, we have a new system dynamic at time $t$ as
\begin{flalign}
\Pi^{\text{($t$+1)}} = \eta^{\text{($t$)}}(\Pi^{\text{($t$)}}, C^{\text{($t$)}}, \mathbf{\Gamma}^{\text{($t$)}}),
\end{flalign}
where $\eta^{\text{($t$)}}$ is the standard non-linear filtering update function (see \cite{nayyar2013decentralized} for more details).

Now, given our framework in Section~\ref{sec:def}, we can prove that the above optimization problem is a networked multi-agent I-POMDP, as shown in the following proposition.
\begin{proposition}
The optimization problem defined in Definition~\ref{def:DCP} is a networked multi-agent I-POMDP.
\end{proposition}
\begin{proof}
We can prove the proposition by defining the networked multi-agent I-POMDP tuple $\langle \mathcal{G}, \{IS_i\}_{i\in [N]},\{A_i\}_{i \in [N]}, \{T_i\}_{i \in [N]}, \{\Omega_i\}_{i \in [N]}, \{O_i\}_{i \in [N]}$, $\{R_i\}_{i \in [N]},M \rangle$ for the optimization problem,
\begin{itemize}
\item Given the definition of shared memory, we can define an equivalent communication network $\mathcal{G}=(\mathcal{V}, \mathcal{E})$. Let the set of $N$ agents (controllers) be the set of nodes $\mathcal{V}$, labeled by index $i = 1,2,...,N$. And an edge $(i,j)$ is in $\mathcal{E}$ if and only if agent $i$ shares memory with agent $j$.
\item For each agent (controller) $i$, the interactive state $IS_i=\mathcal{X} \times M_{ \partial i}$, where $\mathcal{X}$ is the set of states of the physical environment, and $M_{ \partial i}$ is the set of possible models of $i$'s neighbors $ \partial i$.
\item For each agent (controller) $i$, the action space $A^{\text{($t$)}}_i$ at time $t$ is given by the range of the coordinator’s prescription to controller $i$, $\Gamma^{\text{($t$)}}_i$.
\item For each agent (controller) $i$, the transition model $T_i^{\text{($t$)}}$ at time $t$ is given by the dynamic $\Pi^{\text{($t$+1)}} = \eta^{\text{($t$)}}(\Pi^{\text{($t$)}}, C^{\text{($t$)}}, \mathbf{\Gamma}^{\text{(1:$t$-1)}})$.
\item For each agent (controller) $i$, the set of observations $\Omega_i^{\text{($t$)}}$ at time $t$ is given by the set of possible shared memories $\{C^{\text{($t$)}}\}$.
\item For each agent (controller) $i$, the observation function $O_i^{\text{($t$)}}$ at time $t$ is given by $Y^{\text{($t$)}}_i = h^{\text{($t$)}}_i(X^{\text{($t$)}}, W^{\text{($t$)}}_i)$ and $\Gamma^{\text{($t$)}}_i$.
\item For each agent (controller) $i$, the reward function at time $t$ is given by $l(X^{\text{($t$)}}, \mathbf{U}^{\text{($t$)}})$.
\item The message type depends on the definition of shared memory, it could be `action' or `observation'.
\end{itemize}
\end{proof}
\begin{corollary}
The optimal strategy of the optimization problem in Definition~\ref{def:DCP} can be obtained by running Algorithm~\ref{alg:DecBP} and given by Theorem~\ref{thm:OP}.
\end{corollary}

\cite{nayyar2013decentralized} call their solution Dynamic Programming Decomposition. We generalize the scenario to the networked multi-agent case, and may call our solution Belief Propagation Decomposition.
\subsection{Decentralized Spectrum Sharing Problem}\label{subsec:SS}
We consider a contention based decentralized spectrum sharing problem. Given a communication network $\mathcal{G}=(\mathcal{V}, \mathcal{E})$, such that $\mathcal{V}$ is a set of $N$ base stations, labeled by an index $i=1,2,\ldots, N$, and an edge $(i,j)$ is in $\mathcal{E}$ if and only if base station $i$ is backhauled with base station $j$. Each base station serves a given subset of user equipments (UEs). The whole communication network shares a single spectrum, the base stations contend the transmission opportunities in the following way. Each time slot consists of two phases, contention phase and data transmission phase. At contention phase, each base station draws a random number at the start of a time slot, which determines the order of optional transmissions. In the designated slot of the contention phase, a base station can choose to transmit or keep silent. If a base station transmits, it continues transmission through the contention phase and data transmission phase. Ideally, the UE throughput is given by Shannon channel capacity. And the objective of each base station is to maximize the long-term throughput it delivers to its UEs.

For simplicity, we make the following assumptions. We assume each base station serves only one UE. Each base station always has traffic to be delivered to the UE, thus always participates in contention. And there is only downlink traffic. The action space of each base station is $\{\text{transmit}, \text{silent}\}$, which can be denoted as $\{1,0\}$.

Given the above assumptions, we can mathematically formulate the problem. Let us denote the UE served by the base station $i$ by the same index $i$, and so is the link between the UE and the base station. We denote the link strength between base station $i$ and UE $j$ by channel coefficient $h_{ij}$. For each link $i$ in time slot $t$, let us denote the transmission rate by $R_i^{\text{($t$)}}$ and the long-term average rate by $\bar{X}_i^{\text{($t$)}}$.
\begin{flalign}
R_i^{\text{($t$)}} &= W \log_2(1+\text{SINR}_i^{\text{($t$)}}),\\
\bar{X}_i^{\text{($t$)}} &= (1-\frac{1}{B}) \bar{X}_i^{\text{($t$-1)}} + \frac{1}{B} R_i^{\text{($t$)}},
\end{flalign}
where $B > 1$ is a parameter which balances the weights of past and current transmission rates. We denote the actions of all base stations in time slot $t$ as $\mathbf{a}^{\text{($t$)}}=[a_1^{\text{($t$)}}, \ldots, a_N^{\text{($t$)}}]^\intercal \in \{0,1\}^N$. The SINR for UE $i$ is given by
\begin{flalign}
\text{SINR}_i^{\text{($t$)}} = \frac{h_{ii}^{\text{($t$)}}P_t a_i^{\text{($t$)}}}{\sigma^2_{\text{UE}} + \sum_{j\neq i}h_{ji}^{\text{($t$)}}P_t a_j^{\text{($t$)}}}=\frac{S_i^{\text{($t$)}}}{\sigma^2_{\text{UE}} + I_i^{\text{($t$)}}},
\end{flalign}
where $P_t$ is the transmission power, and $\sigma^2_{\text{UE}}$ is the noise power at UE. $S_i^{\text{($t$)}}$ is the signal power for UE $i$ at time $t$, and  $I_i^{\text{($t$)}}$ is the total interference power for UE $i$ at time $t$.

Given the action vector $\mathbf{a}^{\text{($t$)}}$ in each time slot $t$, the long term proportional fairness scheduling utility is \cite{kelly1998rate}
\begin{flalign}
\max_{t \rightarrow \infty} U(\bar{\mathbf{X}}^{\text{($t$)}}) = \max_{t \rightarrow \infty} \sum_{i=1}^N \log(\bar{X}_i^{\text{($t$)}}).
\end{flalign}
And we can split the proportional fairness metric over time by rewriting the utility function up to time slot $T$,
\begin{flalign}
U(\bar{\mathbf{X}}^{\text{($T$)}}) =&  \sum_{i=1}^N \log(\bar{X}_i^{\text{($T$)}}) \nonumber\\
=& \sum_{i=1}^N \log\bigg((1-\frac{1}{B}) \bar{X}_i^{\text{($T$-1)}} + \frac{1}{B} R_i^{\text{($T$)}}\bigg)\nonumber \\
=& \sum_{i=1}^N \log\bigg((1-\frac{1}{B})\bar{X}_i^{\text{($T$-1)}}(1+\frac{R_i^{\text{($T$)}}}{(B-1)\bar{X}_i^{\text{($T$-1)}}})\bigg)\nonumber\\
=& \sum_{i=1}^N (\log \bar{X}_i^{\text{($0$)}} +  \sum_{t=1}^{T}r_i^{\text{($t$)}} ) \nonumber\\
=& \sum_{i=1}^N \log \bar{X}_i^{\text{($0$)}} + \sum_{t=1}^T \sum_{i=1}^N r_i^{\text{($t$)}},
\end{flalign}
where 
\begin{flalign}
r_i^{\text{($t$)}} = \log\bigg((1-\frac{1}{B})\bigg(1+\frac{R_i^{\text{($T$)}}}{(B-1)\bar{X}_i^{\text{($T$-1)}}}\bigg)\bigg).
\end{flalign}

Given our framework, we can prove that the above optimization problem is a
networked multi-agent I-POMDP, as shown in the following proposition.

\begin{proposition}\label{prop:SS}
The decentralized spectrum sharing problem defined in Section~\ref{subsec:SS} is a networked multi-agent I-POMDP.
\end{proposition}
\begin{proof}
We can prove the proposition by defining the networked multi-agent I-POMDP tuple $\langle \mathcal{G}, \{IS_i\}_{i\in [N]},\{A_i\}_{i \in [N]}, \{T_i\}_{i \in [N]}, \{\Omega_i\}_{i \in [N]}, \{O_i\}_{i \in [N]}$, $\{R_i\}_{i \in [N]},M \rangle$ for the optimization problem,
\begin{itemize}
\item The communication network $\mathcal{G}$ is directly defined in the decentralized spectrum sharing problem.
\item For each agent (base station) $i$, the interactive state $IS_i=S \times  M_{ \partial i }$, where $S$ consists of the joint space of the average rate of link $i$ and the channels between all base stations and UE $i$, i.e. $\langle \bar{X}_i^{\text{($t$)}} , \{h_{ji}\}_{j \in [N]}\rangle$ is a state in $S$, and $M_{\partial i}$ is the set of possible models of $i$'s neighbors $\partial i$.
\item For each agent (base station) $i$, the action space $A_i$ is \{\text{transmit}, \text{silent}\}, i.e., $\{0,1\}$.
\item For each agent (base station) $i$, the transition model $T_i: S \times A_i \times A_{\partial i} \times S \rightarrow [0,1]$ is given by the dynamic 
\begin{flalign*}
\bar{X}_i^{\text{($t$)}} &= (1-\frac{1}{B}) \bar{X}_i^{\text{($t$-1)}} + \frac{1}{B} R_i^{\text{($t$)}},
\end{flalign*}
and the channel fading model.
\item For each agent (base station) $i$, at time slot $t$, the observation consists of the average throughput $\bar{X}_i^{\text{($t$-1)}}$, the signal power $S_i^{\text{($t$-1)}}$ and the total interference power $I_i^{\text{($t$-1)}}$ of the previous time slot, i.e., $o_i^t = [\bar{X}_i^{\text{($t$-1)}}, S_i^{\text{($t$-1)}}, I_i^{\text{($t$-1)}}]$.
\item For each agent (base station) $i$, the observation function $O_i$ is directly given by the definitions of involving parameters in $S$, $A_i\times A_{\partial i}$ and $\Omega_i$
\item For each agent (controller) $i$, at time slot $t$, the reward function is given by $r_i^{\text{($t$)}}$.
\item The message type in this problem can be either `action' or `observation'.
\end{itemize}
\end{proof}
The Proposition~\ref{prop:SS} naturally leads to the following corollary,
\begin{corollary}
The optimal strategy of decentralized spectrum sharing problem can be obtained by running Algorithm~\ref{alg:DecBP} and given by Theorem~\ref{thm:OP}.
\end{corollary}
\section{Conclusion}\label{sec:con}
In this paper, we address the problem of multi-agent I-POMDPs with networked agents. In particular, we consider the fully decentralized setting where each agent makes individual decisions and receives local rewards, while exchanging information with neighbors over the network to accomplish optimal network-wide averaged return. Within this setting, we propose a decentralized belief propagation algorithm. We provide theoretical analysis on the convergence of the proposed algorithm. And we show our framework can be applied to various applications. An interesting direction of future research is to extend our algorithms and analyses to the policy gradient methods.
\bibliographystyle{alpha}
\bibliography{sample-base}










\end{document}